\documentclass{article} % For LaTeX2e
\usepackage{iclr2024_conference,times}

\usepackage{amsmath,amsfonts,bm}

\def\eqref#1{equation~\ref{#1}}
\def\1{\bm{1}}

\DeclareMathAlphabet{\mathsfit}{\encodingdefault}{\sfdefault}{m}{sl}
\SetMathAlphabet{\mathsfit}{bold}{\encodingdefault}{\sfdefault}{bx}{n}

\usepackage[colorlinks=true, citecolor=blue]{hyperref}
\usepackage{makecell}
\usepackage{multirow} 
\usepackage{url}            % simple URL typesetting
\usepackage{booktabs}       % professional-quality tables
\usepackage{amsfonts}       % blackboard math symbols
\usepackage{nicefrac}       % compact symbols for 1/2, etc.

\usepackage{microtype}      % microtypography
\usepackage{thmtools}      % for restatable environment
\usepackage{colortbl}

\usepackage{graphicx}
\usepackage{amsmath}
\usepackage{amssymb}
\usepackage{subfigure}
\usepackage{caption} 
\usepackage{natbib}
\usepackage{enumitem}
\usepackage{pdfpages}
\usepackage{tabularray}

\usepackage{amsmath}
\usepackage{amssymb}
\usepackage{mathtools}
\usepackage{amsthm}
\usepackage{thmtools,thm-restate}

\theoremstyle{plain}
\newtheorem{theorem}{Theorem}[section]

\newtheorem{lemma}[theorem]{Lemma}

\theoremstyle{definition}

\theoremstyle{remark}

\usepackage[capitalize,noabbrev]{cleveref}

\usepackage{algorithm}  
\usepackage{algorithmicx}  
\usepackage{algpseudocode}
\usepackage{xcolor}

\usepackage{listings}

\usepackage{subcaption} 

\title{Quality-constrained Entropy Maximization Policy Optimization for LLM diversity}

\author{Haihui Pan, Yuzhong Hong, Kaichen Zhang, Shaoke Lv, Junwei Bao, Hongfei Jiang, Yang Song \\
Zuoyebang Education Technology\\
\texttt{\{panhaihui,lvshaoke,jianghongfei,songyang\}@zuoyebang.com} \\
\texttt{eugene.h.git@gmail.com} \\
\texttt{kzhangbi@connect.ust.hk} \\
\texttt{baojunwei001@gmail.com} \\
}

\begin{document}

\maketitle

\begin{abstract}
In many large language model (LLM) alignment applications, users expect not only high-quality outputs but also substantial diversity. However, existing methods often face a fundamental trade-off between these objectives: approaches that improve output quality tend to reduce diversity, while methods that increase diversity often do so at the expense of quality.
In this work, we propose Quality-constrained Entropy Maximization Policy Optimization (QEMPO), a novel framework that enhances the diversity of LLM outputs while explicitly preserving output quality. QEMPO is grounded in a strong theoretical foundation: we derive a closed-form analytical solution that provably maximizes entropy—a principled measure of diversity—subject to a quality constraint, with guarantees on optimality under the defined objective. Leveraging this solution, QEMPO naturally supports both online and offline training settings.
Empirical results demonstrate that QEMPO consistently improves output diversity without sacrificing quality, and in many cases yields gains in both dimensions compared to existing baselines, aligning with our theoretical guarantees.

\end{abstract}

\section{Introduction} 
% Large language models (LLMs) have demonstrated remarkable capabilities in natural language generation, achieving human-level performance in tasks such as language understanding and commonsense reasoning\citep{Qwen3,deepseek-r1,dubey2024llama}. Despite their powerful generative abilities, a critical issue has rapidly gained attention among researchers: output diversity \citep{li2025preserving,murthy2024one,padmakumar2023does,slocum2025diverse,sun2025curiosity,xu2025understanding,kirk2023understanding}. Generative diversity refers to the capacity of an LLM to produce text that captures the nuances and variations inherent in human expression. In many applications—particularly those without a single "correct" answer, such as creative writing—it is essential for LLMs to generate a broad range of valid and varied text outputs. A lack of diversity may result in monotonous and homogeneous content, ultimately limiting the models' applicability across downstream tasks.
Large language models (LLMs) have demonstrated remarkable capabilities in natural language generation, achieving human-level performance in tasks such as language understanding and commonsense reasoning~\citep{Qwen3,deepseek-r1,dubey2024llama}. Despite these advances, output diversity has emerged as a critical concern~\citep{li2025preserving,murthy2024one,padmakumar2023does,slocum2025diverse,sun2025curiosity,xu2025understanding,kirk2023understanding}. Generative diversity refers to the ability of an LLM to produce varied and nuanced text that reflects the richness of human expression, which is particularly important in open-ended tasks without a single correct answer. Insufficient diversity can lead to a narrow concentration of high-probability outputs, reducing creativity and informativeness while overlooking equally valid alternatives. This not only limits the model’s effectiveness in applications such as creative generation, but may also reinforce dominant patterns in the data, hindering robustness and generalization across diverse scenarios.

% After pre-training, LLMs typically use alignment methods to enhance their instruction-following capabilities. However, recent research indicates that while alignment methods—including Supervised Fine-Tuning (SFT), Reinforcement Learning from Human Feedback (RLHF)\citep{rlhf}, and Direct Preference Optimization (DPO) \citep{dpo}—significantly improve the quality, usefulness, and safety of model outputs, they are widely observed to reduce the diversity of the content generated by LLMs \citep{li2025preserving, slocum2025diverse, xu2025understanding}. For instance, RLHF tends to produce outputs that are more detectable, lengthy, and repetitive \citep{xu2025understanding}. It has been observed to decrease the entropy of the output distribution \citep{padmakumar2023does} and flatten conceptual diversity relative to the base model. 
However, existing LLM training paradigms often exhibit a fundamental trade-off between generation quality and diversity. Approaches such as Supervised Fine-Tuning (SFT), Reinforcement Learning from Human Feedback (RLHF) \citep{rlhf}, and Direct Preference Optimization (DPO) \citep{dpo} substantially improve output quality—enhancing factual accuracy, coherence, and alignment with human preferences—as well as safety. However, they are widely observed to reduce diversity, often leading to more generic, conservative, or repetitive responses \citep{li2025preserving, slocum2025diverse, xu2025understanding}. In contrast, alternative approaches \citep{shypula2025evaluating, zhang2021trading} promote more varied and creative outputs, but typically at the cost of reduced reliability and overall quality.

\begin{figure}[t]
    \centering
    \includegraphics[height=0.28\textwidth]{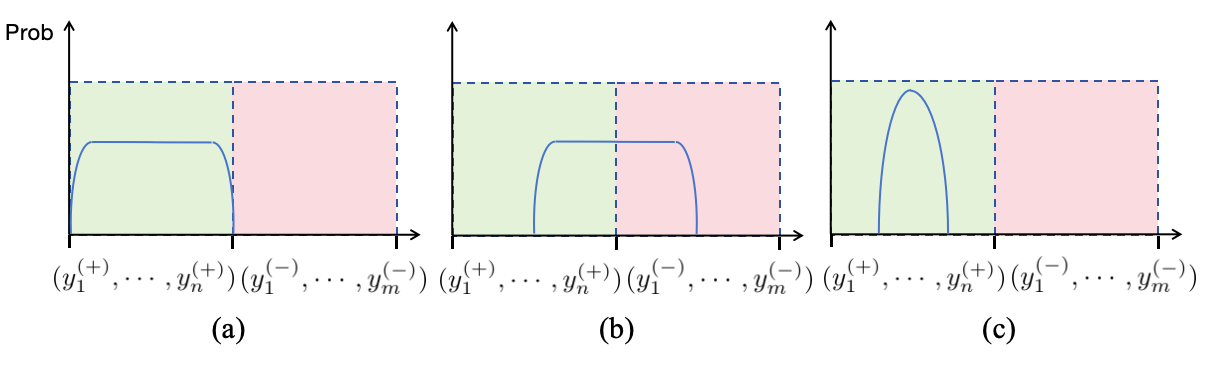}
    \caption{A schematic diagram of the output space distributions of three policies with varying quality and diversity. Here, \((y^{+}_{1}, \cdots, y^{+}_{n})\) represents the set of outputs that align with our preferences, while \((y^{-}_{1}, \cdots, y^{-}_{m})\) denotes the set of outputs that do not align with our preferences. (a) represents a policy whose outputs exhibit high diversity and very high quality; (b) represents a policy whose outputs exhibit high diversity but lower quality; (c) represents a policy whose outputs have limited diversity but possess high quality.}
    \label{fig:prob}
\end{figure}

% In this work, we revisit the roles of diversity and quality in alignment tasks. \cref{prop:one} demonstrates that alignment tasks can be decomposed into two components: maximizing diversity and maximizing quality. This means that both the quality and diversity of the output should be considered in alignment. Soling entropy maximization or quality maximization may prevent the alignment objective from reaching its optimal state. Although a higher entropy of the policy implies greater diversity in the outputs, exclusively focusing on entropy maximization could lead to a significant degradation in output quality. 
% %
% As shown in \cref{fig:prob}, policies (a) and (b) have the same entropy; however, the output quality of policy (a) is clearly higher since all its outputs are satisfactory. Conversely, if we pursue only quality maximization, we may obtain a policy with an output distribution like that of policy (c) in \cref{fig:prob}. Although all outputs from policy (c) are satisfactory, they might be overly concentrated in a specific region, thereby reducing diversity.

We first revisit the interplay between diversity and quality in alignment tasks. The alignment objective can be viewed as balancing two complementary requirements: generating high-quality outputs while maintaining sufficient diversity. Although these two aspects are closely related, they are not equivalent, and optimizing only one of them is generally insufficient.

Diversity is often linked to the entropy of the policy, as higher entropy implies a more spread-out output distribution. However, entropy alone does not capture whether the generated outputs are desirable. As illustrated in \cref{fig:prob}, policies (a) and (b) have the same entropy, yet policy (a) is clearly preferable because all its outputs are satisfactory, whereas policy (b) assigns probability mass to lower-quality responses. This shows that entropy cannot distinguish between distributions with similar diversity but different quality. On the other hand, focusing exclusively on quality can also be problematic. As shown by policy (c) in \cref{fig:prob}, a policy may produce only satisfactory outputs but concentrate them within a narrow region of the output space. While such a policy achieves high quality, its lack of diversity limits its expressiveness and robustness.

Taken together, these observations highlight that neither entropy maximization nor quality maximization alone is sufficient. Effective alignment requires jointly considering both aspects, ensuring that the policy produces outputs that are not only consistently high-quality but also sufficiently diverse.

In this paper, we propose Quality-constrained Entropy Maximization Policy Optimization (QEMPO), a novel method that can enhance the diversity of model outputs while ensuring their quality during the alignment process. QEMPO, along with its variant QEMPO-KL, is grounded in a principled information-theoretic approach: we derive closed-form analytical solutions that maximize policy entropy, subject to explicit quality and KL constraints. We further provide theoretical guarantees showing that QEMPO achieves higher entropy than RLHF under specific conditions. To facilitate practical implementation, we develop both online and offline optimization objectives. Empirical evaluations across multiple base models and tasks demonstrate that QEMPO consistently enhances output diversity across lexical, semantic, and syntactic dimensions without sacrificing quality, often yielding gains in both. The specific contributions of this work are as follows:

% Specifically, the contributions of this work are as follows:
% \begin{itemize} [leftmargin=0.5cm]
%     \item We prove that diversity and quality are two indispensable components in alignment tasks. Furthermore, we prove that Policy Gradient methods primarily optimize only the quality aspect of alignment.
%     \item  We establish that the analytical solution to the quality-constrained KL minimization problem shares the same functional form as the optimal policy derived from RLHF.
%     \item We propose QEMPO and QEMPO-KL to enhance policy diversity while maintaining output quality.
%     \item We prove that under specific conditions, the policy derived from QEMPO achieves higher entropy than that of QEMPO-KL, while QEMPO-KL yields higher entropy than the RLHF policy.
%     \item For both QEMPO and QEMPO-KL, we formulate offline and online optimization objectives. Experimental results validate the effectiveness of the proposed methods.
% \end{itemize}
%

\begin{itemize}[leftmargin=0.5cm]
    \item \textbf{Theoretical Insights:} We prove that quality and diversity are indispensable alignment components and show that Policy Gradient primarily optimizes only for quality.
    
    \item \textbf{Optimal Policy Analysis:} We establish that the analytical solution for quality-constrained KL minimization shares the same functional form as the RLHF optimal policy.
    
    \item \textbf{QEMPO Framework:} We propose QEMPO and QEMPO-KL to explicitly maximize policy entropy and diversity while maintaining rigorous quality constraints.
    
    \item \textbf{Entropy Hierarchy:} We prove that, under specific conditions, the entropy of the resulting policies follows the order: $\pi_{QEMPO} \ge \pi_{QEMPO-KL} \ge \pi_{RLHF}$.
    
    \item \textbf{Algorithmic Validation:} We develop both online and offline optimization objectives for our methods. Experimental results confirm significant diversity gains across lexical, semantic, and syntactic dimensions without sacrificing quality.
\end{itemize}

\section{Preliminary}

\subsection{RLHF and DPO}

\textcolor{black}{\textbf{Notation}. We denote random variables by capital letters and their realizations by lowercase letters. Let \( X \) represent the input prompt to the LLM, and \( Y \) denote the output. We denote the policy or LLM by $\pi(y|x)$, which can thus be regarded as a distribution. We use $r(x, y)$ to represent the reward model}.

\textbf{RLHF}. RLHF primarily consists of three steps. First, a model is fine-tuned to serve as the reference policy, denoted as $\pi_{\mathrm{ref}}(x|y)$. Subsequently, human-labeled preference data is collected to train a reward model. Finally, RLHF uses the trained reward model and the reference policy to optimize the following objective:
\begin{align}
\label{eq-rlhf}
    \max_{\pi} \mathbb{E}_{x\sim D,y\sim \pi}[r(x,y)]- \beta \mathrm{KL}(\pi(y|x) \parallel \pi_{\mathrm{ref}}(y|x) ).
\end{align}

\textbf{DPO}. \citep{dpo} first transforms \cref{eq-rlhf} to an equivalent minimization form. Subsequently, by leveraging the non-negativity of the KL divergence, it derives the following closed-form expression for the policy $\pi(x|y)$ at the minimum:
\begin{align}
    \pi(y|x) =  \pi_{\mathrm{RLHF}}(y|x) = \frac{1}{Z(x)}\pi_{\mathrm{ref}}(y|x)\exp(\frac{1}{\beta}r(x,y)).
\end{align}
Based on the equation above, $r(x,y)$ can be expressed as:
\begin{align}
r(x,y) = \beta \log \frac{\pi(y|x)}{\pi_{\mathrm{ref}}(y|x)}+\beta \log Z(x). 
\end{align}
Substituting $r(x,y)$ into the Bradley-Terry model and performing maximum likelihood estimation yields the DPO objective function:
\begin{align}
    \max_{\pi} \mathbb{E}_{y_{w} \succ  y_{l}\sim D}[\log \sigma ( \beta \log \frac{\pi(y_{w}|x)}{\pi_{\mathrm{ref}}(y_{w}|x)}-\beta \log \frac{\pi(y_{l}|x)}{\pi_{\mathrm{ref}}(y_{l}|x)})]
\end{align}

\subsection{Quality-constrained KL Minimization}
\textcolor{black}{As a conceptual exercise, we can formulate a problem that minimizes KL divergence from the reference policy subject to a quality constraint.} As shown in \cref{prop:kl-rlhf}, when \(1/\beta = \lambda\), the analytical solution to this constrained problem coincides in form with that of RLHF. This formal equivalence reveals an important insight: the RLHF objective can be satisfied by a policy that merely concentrates probability mass on a narrow subset of high-reward responses while remaining close to \(\pi_{\mathrm{ref}}\) in the KL sense. It does not inherently require the policy to explore or cover the full range of acceptable outputs. Consequently, RLHF often converges to a low-diversity policy, as illustrated in \cref{fig:prob}-(c). The proof of \cref{prop:kl-rlhf} is in the Appendix \ref{proof:p1}.
\begin{restatable}{proposition}{propone}
\label{prop:kl-rlhf} 
For the optimization problem: 
\begin{align}
\min_{\pi} \mathrm{KL}[\pi(y|x) || \pi_{\mathrm{ref}}(y|x)] \quad s.t. \quad \mathbb{E}_{\pi(y|x)}[r(x,y)] \ge \mathrm{R}
\nonumber
\end{align}
where $\mathrm{R} = \mathbb{E}{\pi^{}_{\mathrm{RLHF}}}[r(x,y)]$. The analytical solution that minimizes this optimization objective is 
$\pi (y|x)
   =   \frac{\pi_{\mathrm{ref}}(y|x) \exp( \lambda r(x,y)) }{Z(x)} $ where $Z(x) = \sum_{y}\pi_{\mathrm{ref}}(y|x) \exp( \lambda r(x,y))$ and \textcolor{black}{$\lambda$ is the Lagrange multiplier for this optimization objective}.

\end{restatable}

\section{Quality-constrained Entropy Maximization Policy Optimization}

\subsection{Quality and Diversity in Alignment}
In this subsection, we revisit the roles of quality and diversity in alignment tasks. We categorize the output into two subsets: \( Y^{+} \) consists of outputs that satisfy the given prompt \( X \), while \( Y^{-} \) comprises outputs that do not fulfill \( X \). Let \( \pi^*(y|x) \) represent our ideal policy. From a practical standpoint, we aim for \( \pi^*(y|x) \) to consistently generate outputs that meet the requirements of the given prompt while maximizing output diversity as much as possible. Specifically, \( \pi^*(y|x) \) should possess the following two fundamental properties:

\textbf{Quality}: $\sum_{y\in Y^{+}}\pi^{*}(y|x)  = 1- \epsilon$ where $\epsilon$ is a sufficiently small number.

\textbf{Diversity}: $\pi^{*}(y^{+}_{i} \mid  x) =  \pi^{*}(y^{+}_{j} \mid x)$, $\forall y^{+}_{i},y^{+}_{j} \in \mathrm{Y}^{+}$.

Let $\pi(y|x)$ denote the policy to be aligned. The alignment objective can then be defined as:
\begin{align}
    \min_{\pi} \mathrm{KL}(\pi(y|x) || \pi^{*}(y|x)). 
\end{align}
Specifically, we use the entropy of the policy, $H_{\pi}(Y|X)$, to represent diversity, and use $\sum_{y\in Y^{+}}\pi(y|x)$ to represent the quality of the policy output. From the perspective of information theory, entropy represents the uncertainty of the output, meaning that the probability of the output should not be concentrated in specific areas, which aligns with the goal of diversity.
% Using entropy to represent diversity aligns with our above definition of diversity. When the probabilities of the output samples are more similar, the output is less concentrated on specific samples, and the entropy of the output increases.
\cref{prop:one} indicates that to achieve the ideal state of $\pi(y|x)$, both the output diversity and quality of $\pi(y|x)$ should be improved, which also suggests that for alignment tasks, diversity and quality should be considered simultaneously. 

It is important to note that if we excessively increase the model's output probability on a small subset of $Y^{+}$, it is possible to achieve a relatively large value for $\sum_{y\in Y^{+}}\pi(y|x)$, thereby improving quality. However, in this case, diversity will inevitably decrease, making it impossible to achieve the optimal alignment effect. We further demonstrate that optimization objectives like Policy Gradient \citep{pg} inherently only optimize the quality in alignment tasks. A potential drawback of such an optimization objective is that if the policy focuses solely on a single \( y^{+} \) during the learning process, it can still achieve the effect of maximizing the optimization objective, but this is not what we actually desire. The proof of \cref{prop:one} is provided in Appendix \ref{proof:2}, and the proof of Corollary \ref{cly:one} is provided in Appendix \ref{Corollary:3.1}.

\begin{restatable}{proposition}{propOne}
\label{prop:one}
Let \(\pi(y|x)\) denote the policy to be aligned and \(\pi^*(y|x)\) the ideal policy. Furthermore, by the principle of maximum entropy, we also assume \(\pi^{*}\) is uniform over \(Y^{-}\) in the absence of any discriminating information. The alignment objective \(\mathrm{KL}(\pi(y|x) \parallel \pi^*(y|x))\) can be minimized by jointly maximizing quality and diversity.
\end{restatable}
\begin{restatable}{corollary}{clyOne}
\label{cly:one}
For Policy Gradient, if the reward function is defined as $r(x,y) =\left\{\begin{matrix}
 1 & y \in Y^{+}\\
 0 & y \in Y^{-}
\end{matrix}\right.$, then Policy Gradient only optimizes the quality component in alignment tasks.
\end{restatable}

\subsection{Quality and KL-constrained Entropy Maximization}
%Whether from the perspective of the original optimization objective of RLHF or the optimization objective in \cref{prop:kl-rlhf}, RLHF itself does not explicitly encourage diversity in generation.
Unlike the optimization objective in \cref{prop:kl-rlhf}, we directly replace the minimizing of KL divergence with maximizing the entropy to explicitly encourage diversity in generation, while treating KL divergence and quality as constraints. \cref{prop:kl-qempo} provides the analytical solution form of the corresponding optimization problem. The difference between the analytical solution of RLHF and that of \cref{prop:kl-qempo} lies in the addition of an exponential term to $\pi_{\mathrm{ref}}(y|x)$. Since $0 < \frac{\lambda_{2}}{\lambda_{2}+1} < 1$, the amplifying effect of the exponential term becomes greater when $\pi_{\mathrm{ref}}(y|x)$ is smaller. For example, $0.9^{0.5} \approx 0.9487$, with an amplification factor of 1.05, while $0.3^{0.5} \approx 0.5477$, with an amplification factor of 1.83.  Compared to not adding the exponential term, the effect of $\frac{\lambda_{2}}{\lambda_{2}+1}$ is to make the output distribution more uniform. Specifically, \cref{prop:entropy-kl} shows that when $\frac{1}{\beta} = \frac{\lambda_{1}}{\lambda_{2}}$, $\pi_{\mathrm{QEMPO-KL}}$ has higher entropy than $\pi_{\mathrm{RLHF}}(y|x)$. The proofs of \cref{prop:kl-qempo} and \cref{prop:entropy-kl} are provided in Appendix \ref{p3} and Appendix \ref{p4}, respectively. 

\begin{restatable}{proposition}{propTwo}
\label{prop:kl-qempo} 
For the optimization problem: $ \max_{\pi} H_{\pi}(Y|X) \quad s.t.  \left\{\begin{matrix}
\mathbb{E}_{\pi(y|x)}[r(x,y)] \ge \mathrm{R} 
\\
\mathrm{KL}(\pi_{}\parallel \pi_{\mathrm{ref}} )\le  \mathrm{K}
\end{matrix}\right.$
where $\mathrm{R} = \mathbb{E}_{\pi_{\mathrm{RLHF}}}[r(x,y)], \mathrm{K}= \mathrm{KL}(\pi_{\mathrm{RLHF}}\parallel \pi_{\mathrm{ref}} )$. The analytical solution that minimizes this optimization objective is 
$  \pi_{\mathrm{QEMPO-KL}}
  =  \frac{\pi_{\mathrm{ref}}^{\frac{\lambda_{2}}{\lambda_{2}+1}} \exp( \frac{\lambda_{1}}{\lambda_{2}+1} r(x,y))}{Z(x)}  $ where
  \textcolor{black}{$\lambda_1$ and $\lambda_2$ are the Lagrange multipliers for this optimization objective} and $Z(x) = \sum_{y}\pi_{\mathrm{ref}}^{\frac{\lambda_{2}}{\lambda_{2}+1}} \exp( \frac{\lambda_{1}}{\lambda_{2}+1} r(x,y))$.
\end{restatable}

\begin{restatable}{proposition}{propFour}
\label{prop:entropy-kl} 
When $\frac{1}{\beta} = \frac{\lambda_{1}}{\lambda_{2}}$, it satisfies: $H_{\pi_{\mathrm{QEMPO-KL}}}(Y|X) \ge  H_{\pi_{\mathrm{RLHF}}}(Y|X)$.
\end{restatable}

\subsection{Quality-constrained Entropy Maximization}

By replacing the optimization objective in \cref{prop:kl-rlhf} with maximizing policy entropy while treating KL divergence and quality as constraints, we obtain $\pi_{\mathrm{QEMPO-KL}}(y|x)$, which possesses greater entropy than $\pi_{\mathrm{RLHF}}(y|x)$. However, from a constraint perspective, if we remove the KL constraint in \cref{prop:kl-qempo}, the policy's solution space expands, theoretically implying that the resulting policy should exhibit even greater entropy than $\pi_{\mathrm{QEMPO-KL}}(y|x)$. Moreover, \cref{prop:one} indicates that quality and diversity are the essential components in alignment tasks. Therefore, we further remove the KL constraint in \cref{prop:kl-qempo}. \cref{prop:qempo} provides the policy $\pi_{\mathrm{QEMPO}}(y|x)$ under the corresponding optimization objective. 

It can be observed that compared to $\pi_{\mathrm{QEMPO-KL}}(y|x)$, $\pi_{\mathrm{QEMPO}}(y|x)$ removes $\pi_{\mathrm{ref}}(y|x)$. \textcolor{black}{Whether removing $\pi_{\mathrm{ref}}(y|x)$ brings benefits often depends on the quality of $\pi_{\mathrm{ref}}(y|x)$ itself. If the performance of $\pi_{\mathrm{ref}}(y|x)$ is sufficiently high, this implies that the influence of $\pi_{\mathrm{ref}}(y|x)$ on $\pi_{\mathrm{QEMPO-KL}}(y|x)$ can be positive.} However, when the performance of $\pi_{\mathrm{ref}}(y|x)$ is inadequate—for instance, when a sample is not the desired one but $\pi_{\mathrm{ref}}(y|x)$ still assigns a high probability—the negative impact on $\pi_{\mathrm{QEMPO-KL}}(y|x)$ can be relatively significant. Nevertheless, from a practical perspective, removing $\pi_{\mathrm{ref}}(y|x)$ reduces memory usage, which is beneficial for the training process. The proofs of \cref{prop:qempo} and \cref{prop:entropy-kl-2} are provided in Appendix \ref{p5} and Appendix \ref{p6}, respectively.

\begin{restatable}{proposition}{propFive}
\label{prop:qempo} 
For the optimization problem: 
\begin{align}
\max_{\pi} H_{\pi}(Y|X) \quad s.t. \quad\mathbb{E}_{\pi(y|x)}[r(x,y)] \ge \mathrm{R}
\nonumber
\end{align}
where $\mathrm{R} = \mathbb{E}{\pi^{}_{\mathrm{RLHF}}}[r(x,y)]$. The analytical solution that minimizes this optimization objective is 
$ \pi_{\mathrm{QEMPO}}(y|x)
  =  \frac{ \exp( \lambda r(x,y)) }{Z(x)}$ where $Z(x)= \sum_{y}  \exp( \lambda r(x,y))$ and \textcolor{black}{$\lambda$ is the Lagrange multiplier for this optimization objective}.
\end{restatable}

\begin{restatable}{proposition}{propSix}
\label{prop:entropy-kl-2} 
When \(\lambda \leq \frac{\lambda_1}{\lambda_2 + 1}\) and \(\frac{\lambda_2}{\lambda_1} \lambda\) is sufficiently small, it satisfies: \(H_{\pi_{\mathrm{QEMPO}}}(Y|X) \geq H_{\pi_{\mathrm{QEMPO-KL}}}(Y|X)\).
\end{restatable}

\subsection{Optimization of QEMPO}
In \cref{tab:policy}, we present the forms of the optimal policy solutions under RLHF, QEMPO, and QEMPO-KL. Based on the forms of the optimal policies, we can transform the reward function \( r(x, y) \) and derive the corresponding expressions. To optimize QEMPO and QEMPO-KL, we use offline and online modes respectively for optimization.

\textbf{Offline mode.} Similarly to DPO, we can substitute the reward expressions corresponding to QEMPO and QEMPO-KL into the BT model to obtain their offline optimization objectives. For QEMPO-KL, the optimization objective is:
\begin{align}
    \max_{\pi} \mathbb{E}_{y^{+} \succ  y^{-}\sim D}[\log \sigma ( \frac{1}{\lambda_{1}}  \ln \pi(y^{+}|x)+ \frac{\lambda_{2}}{\lambda_{1}} \ln \frac{\pi(y^{+}|x)}{\pi_{\mathrm{ref}}(y^{+}|x)}- \frac{1}{\lambda_{1}}  \ln \pi(y^{-}|x)- \frac{\lambda_{2}}{\lambda_{1}} \ln \frac{\pi(y^{-}|x)}{\pi_{\mathrm{ref}}(y^{-}|x)})].
\end{align}
For QEMPO, the optimization objective is:
\begin{align}
    \max_{\pi} \mathbb{E}_{y^{+} \succ  y^{-}\sim D}[\log \sigma (\frac{1}{\lambda} \ln \pi(y^{+}|x)-\frac{1}{\lambda} \ln \pi(y^{-}|x))].
\end{align}

\textbf{Online mode.} The gradient-weighting method introduced by \citep{gvpo} incorporates reward expressions from optimal policy solutions. By substituting the reward expressions of the QEMPO-KL and QEMPO optimal solutions into this method, we can derive the optimization objectives for QEMPO-KL and QEMPO. For QEMPO-KL, the optimization objective is:
\begin{align}
    2(\frac{\lambda_{2}+1}{\lambda_{1}}) \mathbb{E}_{x,y}[(r(x,y)-\mathbb{E}[r(x,y)])\log \pi+ \frac{\lambda_{2}}{\lambda_{1}} \mathrm{Cov}(\log \pi, \log \pi_{\mathrm{ref}})-0.5(\frac{\lambda_{2}+1}{\lambda_{1}}) \mathrm{Var}(\log \pi)]
\end{align}
where \( r(x, y) \) is the reward function used during training, \( \operatorname{Cov}(\log \pi, \log \pi_{\mathrm{ref}}) \) denotes the covariance between \( \log \pi(y|x) \) and \( \log \pi_{\mathrm{ref}}(y|x) \), and \( \operatorname{Var}(\log \pi) \) is the variance of \( \log \pi(y|x) \). For QEMPO, the optimization objective is: \begin{align}
    2\frac{1}{\lambda} \mathbb{E}_{x,y}[(r(x,y)-\mathbb{E}[r(x,y)])\log \pi(y|x)-0.5\frac{1}{\lambda}  \mathrm{Var}(\log \pi(y|x))]
\end{align}

For both QEMPO and QEMPO-KL in online mode, to ensure quality-constrained policy optimization while maximizing entropy, we optimize entropy only when the policy produces correct responses to questions. Consequently, when policy outputs contain errors, we remove the variance term for such data points in the optimization objective to prevent entropy increase under low-quality outputs.

\begin{table}[h]
\caption{The closed-form solution of the optimal policy for RLHF, QEMPO and QEMPO-KL, and the re-expression of the reward function under the optimal policy.}
\label{tab:policy}
\begin{tabular}{l|l|l}
\hline
 & \multicolumn{1}{c|}{Optimal policy} & \multicolumn{1}{c}{Reward expression} \\ \hline
\multirow{2}{*}{RLHF} & \multirow{2}{*}{$\pi (y|x)= \frac{\pi_{\mathrm{ref}}(y|x) \exp( \frac{1}{\beta}r(x,y))}{Z(x)} $} & \multirow{2}{*}{$r_{\theta}(x,y) = \beta \ln \frac{\pi(y|x)}{\pi_{\mathrm{ref}}(y|x)}+ \beta \ln Z(x)$} \\
 &  &  \\ \hline
\multicolumn{1}{c|}{\multirow{4}{*}{QEMPO-KL}} & \multirow{4}{*}{$\pi (y|x)=\frac{\pi_{\mathrm{ref}}(y|x)^{\frac{\lambda_{2}}{\lambda_{2}+1}} \exp( \frac{\lambda_{1}}{\lambda_{2}+1} r(x,y)) }{Z(x)} $} & \multirow{4}{*}{\makecell{$r_{\theta}(x,y) =  \frac{1}{\lambda_{1}}  \ln \pi(y|x)+ $ \\$\frac{\lambda_{2}}{\lambda_{1}} \ln \frac{\pi(y|x)}{\pi_{\mathrm{ref}}(y|x)} +  \frac{\lambda_{2}+1}{\lambda_{1}}  \ln Z(x)$}} \\
\multicolumn{1}{c|}{} &  &  \\
\multicolumn{1}{c|}{} &  &  \\
\multicolumn{1}{c|}{} &  &  \\ \hline
\multirow{2}{*}{QEMPO} & \multirow{2}{*}{$\pi(y|x) =  \frac{ \exp( \lambda r(x,y))}{Z(x)}  $} & \multirow{2}{*}{$r_{\theta}(x,y)  =  \frac{1}{\lambda} \ln \pi(y|x)+ \frac{1}{\lambda} \ln Z(x)$} \\
 &  &  \\ \hline
\end{tabular}
\end{table}

\section{Experiments}
\label{sec:setup}

\subsection{Experiments in Offline Mode}
\subsubsection{Setup}
%To validate the effectiveness of QEMPO,
\textbf{Base model and training data.}  For smaller models, we used Qwen2.5-1.5B-Instruct \citep{qwen2.5} and Llama-3.2-1B-Instruct \citep{llama3.2}. For larger models, we used Qwen2.5-7B-Instruct \citep{qwen2.5} and Llama-3.1-8B-Instruct \citep{dubey2024llama}. We use UltraFeedback Binarized dataset \citep{ultrafeedback} for model training. 

\textbf{Implementation details in offline mode.}
For all experiments, we used a batch size of 128 and trained for 1 epoch. We use Adam \citep{adam} as the optimizer with a warmup ratio of 1. For all experiments involving Llama-3.2-1B-Instruct and Qwen2.5-1.5B-Instruct, the learning rate is set to 5.0e-7. For all experiments involving Llama-3.1-8B-Instruct and Qwen2.5-7B-Instruct, the learning rate is set to 1.0e-7. For RLHF, we follow the setup in \citep{zephyr} and set $\beta$ to 1e-2. For QEMPO, we set $\frac{1}{\lambda}$ to 4e-3. For QEMPO-KL, {we set $\frac{1}{\lambda_{1}}$ to  4e-3  and  set $\frac{\lambda_{2}}{\lambda_{1}}$ to 1e-2.}  We use the test set loss of the trained model on UltraFeedback Binarized as the criterion for selecting the best model.

\textbf{Evaluation metrics.}
For the quality of LLM outputs, we use gpt-4o-2024-08-06 to evaluate and score it. For the diversity of LLM outputs, we use the diversity evaluation framework proposed by \citep{metric} to assess it. 
% \citep{metric} evaluates the diversity of outputs from lexical, syntactic, and semantic dimensions. 
%The lexical  diversity is calculated as the ratio of unique n-grams to the total number of n-grams. To calculate syntactic diversity, \citep{metric} uses a neural parser to generate dependency trees from sentences, then the Weisfeiler-Lehman graph kernel \citep{grakel} is used to compute the average similarity among the dependency trees. Finally, subtracting the computed average similarity from 1 yields the final syntactic diversity. For semantic diversity, \citep{metric} uses Sentence-BERT \citep{reimers2019sentence} to convert the output into embeddings and calculates the average similarity among the embeddings. Finally, subtracting the computed average similarity from 1 yields the final semantic diversity. Higher values of lexical diversity, syntactic diversity, and semantic diversity indicate higher diversity.

\begin{table}[]
\caption{
Evaluation results of output diversity and quality on MT-Bench. Diversity is measured by first sampling each first-round query 16 times, then using randomly selected first-round outputs as context for a second round of 16 samples, and computing overall diversity (higher is more diverse). Quality is scored by gpt-4o-2024-08-06. Values after ± denote standard deviation across 5 runs.}
\label{tab:diversity}
\centering
\begin{tabular}{
>{\columncolor[HTML]{FFFFFF}}l 
>{\columncolor[HTML]{FFFFFF}}c 
>{\columncolor[HTML]{FFFFFF}}c 
>{\columncolor[HTML]{FFFFFF}}c 
>{\columncolor[HTML]{FFFFFF}}c
>{\columncolor[HTML]{FFFFFF}}c
}
\hline
 & Lexical & Semantic & Syntactic & Avg & MT-Bench \\ \hline
\em Qwen2.5-1.5B-Instruct & 41.27 ± 0.15 & 29.20 ± 0.11 & 53.39 ± 0.18 & 41.29 & 6.00 ± 0.06\\
$\quad \vdash$ SPL  & 43.53 ± 0.14 & 30.09 ± 0.09 & 54.52 ± 0.20 & 42.71  & 6.29 ± 0.06\\
$\quad \vdash$ RLHF & 40.42 ± 0.16 & 28.74 ± 0.10 & 51.92 ± 0.18 & 40.36 & 6.30 ± 0.07\\
$\quad \vdash$ QEMPO & \textbf{44.23} ± 0.15 & \textbf{30.18} ± 0.10 & \textbf{54.29} ± 0.19 & \textbf{42.90} & \textbf{6.54} ± 0.05\\
$\quad \llcorner$ QEMPO-KL & 44.18 ± 0.16 & 30.08 ± 0.11 & 53.30 ± 0.18 & 42.52 & 6.43 ± 0.04 \\ \hline

\em Llama-3.2-1B-Instruct & 40.78 ± 0.16 & 29.54 ± 0.13 & 53.55 ± 0.21 & 41.29 & 5.32 ± 0.03\\
$\quad \vdash$ SPL& 47.88 ± 0.15 & 30.83 ± 0.15 & 52.18 ± 0.17 & 43.63 & 5.15 ± 0.04\\
$\quad \vdash$ RLHF & 43.85 ± 0.16 & 29.08 ± 0.10 & 52.43 ± 0.18 & 41.79 & 5.68 ± 0.03\\
$\quad \vdash$ QEMPO & 46.51 ± 0.13 & \textbf{30.84} ± 0.11 & \textbf{53.95} ± 0.19 & 43.77 & \textbf{5.82} ± 0.03\\
$\quad \llcorner$ QEMPO-KL & \textbf{48.36} ± 0.16 & 30.55 ± 0.11 & 53.38 ± 0.18 & \textbf{44.10} & 5.55 ± 0.04 \\ \hline

\em Qwen2.5-7B-Instruct & 31.68 ± 0.14 & 28.04 ± 0.09 & 52.28 ± 0.19 & 37.33 & 7.78 ± 0.08\\
$\quad \vdash$ SPL & \textbf{36.18} ± 0.14 & \textbf{30.31} ± 0.12 & \textbf{55.77} ± 0.16 & \textbf{40.75} & 7.82 ± 0.07\\
$\quad \vdash$ RLHF & 31.98 ± 0.15 & 27.41 ± 0.13 & 51.48 ± 0.17 & 36.95 & 7.90 ± 0.06\\
$\quad \vdash$ QEMPO & 35.37 ± 0.16 & 28.34 ± 0.09 & 52.45 ± 0.19 & 38.72 & 7.92 ± 0.07\\
$\quad \llcorner$ QEMPO-KL & 34.60 ± 0.15 & 28.08 ± 0.12 & 52.34 ± 0.18 & 38.34 & \textbf{7.96} ± 0.06 \\ \hline

\em Llama-3.1-8B-Instruct & 34.85 ± 0.17 & 29.55 ± 0.13 & 54.12 ± 0.15 & 39.51 & 7.67 ± 0.05\\
$\quad \vdash$ SPL & 31.81 ± 0.14 & 27.41 ± 0.11 & 52.34 ± 0.15 & 37.19 & 6.71 ± 0.03\\
$\quad \vdash$RLHF & 34.50 ± 0.15 & 27.73 ± 0.13 & 51.85 ± 0.17 & 38.03 & 7.80 ± 0.04\\
$\quad \vdash$ QEMPO & \textbf{35.60} ± 0.14 & 28.70 ± 0.13 & 53.42 ± 0.16 & 39.24 & \textbf{7.84} ± 0.05\\
$\quad \llcorner$ QEMPO-KL & 33.63 ± 0.15 & \textbf{30.68} ± 0.10 & \textbf{54.91} ± 0.15 & \textbf{39.74} & 7.31 ± 0.04 \\ \hline
\end{tabular}
\end{table}

\begin{figure}[h]
    \centering
    \includegraphics[height=0.45\textwidth]{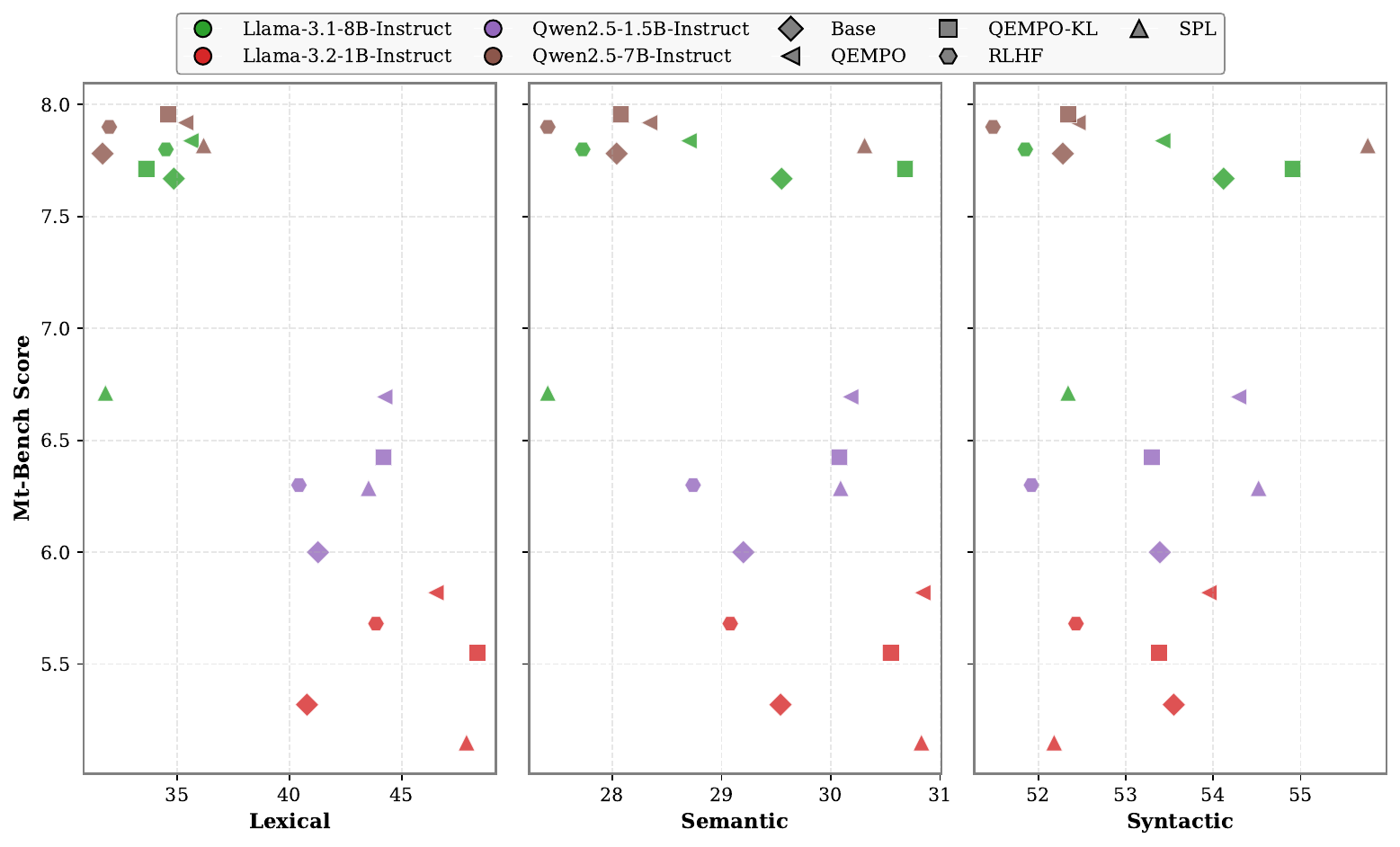}
    \caption{Diversity and quality across different models and methods. Distinct colors represent different base models, while varying shapes denote different approaches.}
    \label{fig:div-qua}
\end{figure}
% As shown in \cref{fig:div-qua}, smaller models tend to exhibit higher diversity but often at the cost of lower quality.
\subsection{Performance}
We evaluate the effectiveness of the proposed methods on MT-Bench, with SPL~\citep{slocum2025diverse} serving as a baseline.
The evaluation results are presented in \cref{tab:diversity}. Overall, RLHF tends to reduce output diversity across most scenarios, consistent with the findings of~\citep{xu2025understanding}. In contrast, both QEMPO and QEMPO-KL consistently improve diversity over RLHF in all experimental settings, demonstrating the effectiveness of our proposed approaches.

SPL also achieves higher diversity than RLHF in most cases, corroborating the findings of~\citep{slocum2025diverse}. However, its performance is highly unstable across different base models. On Llama-3.1-8B-Instruct, SPL yields the lowest diversity among all compared methods, while on both Llama-3.2-1B-Instruct and Llama-3.1-8B-Instruct, its output quality drops below that of the base model, indicating that SPL's diversity gains often come at the cost of quality degradation. We attribute this instability to SPL's decoupled treatment of entropy and cross-entropy, which lacks the joint constraint provided by the KL divergence in RLHF and makes the optimization sensitive to model-specific characteristics such as scale and prior output distribution.

In contrast, QEMPO consistently outperforms RLHF in both diversity and quality across all experimental settings, benefiting from its principled quality-constrained entropy maximization framework. QEMPO-KL yields mixed results: it performs on par with or better than RLHF on the Qwen model series, but slightly underperforms on the Llama series. As \cref{fig:div-qua} reveals, the Llama series exhibits inherently greater diversity than the Qwen series at comparable model sizes, which suggests that QEMPO-KL provides more substantial gains for models with initially lower diversity. Additionally, as shown in \cref{tab:mmlu}, both QEMPO and QEMPO-KL maintain or slightly improve MMLU \citep{mmlu}  scores over their respective base models across all four architectures, confirming that the proposed methods do not negatively affect general model utility.
\subsubsection{Ablation Study}
\begin{table}[]
\caption{Quality and diversity of Qwen2.5-1.5B-Instruct when selecting different hyperparameters on QEMPO and QEMPO-KL. Values after ± denote standard deviation across 5 runs.}
\label{tab:ablation}
\centering
\begin{tabular}{
>{\columncolor[HTML]{FFFFFF}}l 
>{\columncolor[HTML]{FFFFFF}}c 
>{\columncolor[HTML]{FFFFFF}}c 
>{\columncolor[HTML]{FFFFFF}}c cc}
\hline
 & Lexical & Semantic & Syntactic & Average & MT-Bench \\ \hline
 \em QEMPO
 \\
($\frac{1}{\lambda}$=1e-2) & \textbf{45.30 ± 0.13} & \textbf{30.98 ± 0.11} & \textbf{54.46 ± 0.17} & \textbf{43.58} & 6.31± 0.06 \\
($\frac{1}{\lambda}$=6e-3) & 44.78 ± 0.18 & 30.48 ± 0.14 & 54.44 ± 0.11 & 43.23 & 6.34 ± 0.05 \\
 ($\frac{1}{\lambda}$=4e-3) & {44.23} ± 0.15 & {30.18} ± 0.10 & {54.29} ± 0.19 & {42.90} & \textbf{6.54} ± 0.05\\
($\frac{1}{\lambda}$=2e-3) & 42.45 ± 0.16 & 29.25 ± 0.08 & 53.43 ± 0.14 & 41.71 & 6.48 ± 0.05 \\
($\frac{1}{\lambda}$=1e-3) & 42.51 ± 0.11 & 29.74 ± 0.13 & 53.11± 0.12 & 41.79 & 6.30 ± 0.06\\ \hline
 \em QEMPO-KL \\
($\frac{1}{\lambda1}$=4e-3, $\frac{\lambda_{2}}{\lambda_{1}}$=1e-2) & 44.18 ± 0.16 & 30.08 ± 0.11 & 53.30 ± 0.18 & 42.52 & 6.43 ± 0.04 \\ 
($\frac{1}{\lambda1}$=2e-3, $\frac{\lambda_{2}}{\lambda_{1}}$=1e-2) & 43.20 ± 0.12 & 29.45  ± 0.15 & 52.93 ± 0.13& 41.86 & 6.35 ± 0.06 \\
($\frac{1}{\lambda1}$=4e-3, $\frac{\lambda_{2}}{\lambda_{1}}$=6e-3) & 43.95 ± 0.15 & 30.15 ± 0.14 & 53.26 ± 0.11 & 42.45 & \textbf{6.36 ± 0.05}\\
($\frac{1}{\lambda1}$=4e-3, $\frac{\lambda_{2}}{\lambda_{1}}$=1.2e-2) & \textbf{44.90  ± 0.11} & 29.90± 0.13 & 52.69 ± 0.13 & 42.50 & 6.28± 0.04 \\
($\frac{1}{\lambda1}$=6e-3, $\frac{\lambda_{2}}{\lambda_{1}}$=1e-2) & 44.85± 0.12 & \textbf{30.31± 0.15} & \textbf{53.58± 0.14} & \textbf{42.91} & 6.34 ± 0.04 \\ \hline
\end{tabular}
\end{table}

We selected Qwen2.5-1.5B-Instruct to conduct a sensitivity analysis with multiple different hyperparameter settings for both QEMPO and QEMPO-KL. For QEMPO, we set \(\frac{1}{\lambda}\) to \{ 1e-2, 6e-3, 2e-3, 1e-3 $\}$ . For QEMPO-KL, we set \((\frac{1}{\lambda_{1}}, \frac{\lambda_{2}}{\lambda_{1}})\) to  \{ (2e-3, 1e-2), (6e-3, 1e-2), (4e-3, 6e-3), (4e-3, 1.2e-2)  \}. The experimental results are shown in \cref{tab:ablation}. It can be observed that while selecting different hyperparameters has some impact on output diversity and quality, the overall variation is not significant. Compared to RLHF's average diversity of 40.36, all results from QEMPO and QEMPO-KL perform better. Moreover, in terms of quality, all results from QEMPO are no lower than RLHF's score of 6.30, while only one result on QEMPO-KL is slightly lower than RLHF. 

\subsection{Experiments in Online Mode}
\subsection{Setup}
\textbf{Base model and training data.} We selected Qwen2.5-7B-Instruct as the base model. We utilized a combined total of 12,000 samples from the training sets of GSM8K\citep{gsm8k} and MATH500\citep{math500}, each containing a given question and its corresponding answer.

\textbf{Implementation details in online mode.} For all experiments, we set the training batch size to 1024, the number of training epochs to 3, and the learning rate to 1e-7 using a constant learning rate schedule. For each query, we generate 10 responses with the temperature set to 1 during inference. We employ the GRPO method for advantage estimation. For the hyperparameters in RLHF, QEMPO, and QEMPO-KL, our settings are consistent with those in the offline mode. We employ xverify \citep{xverify} to validate whether the reasoning matches the given answer.

% For all experiments, we set the training batch size to 1024, the number of training epochs to 3, and the learning rate to 1e-7 using a constant learning rate schedule. For each query, we generate 10 responses with the temperature set to 1 during inference. We employ the GRPO method for advantage estimation. For RLHF, the KL coefficient is set to 0.01. For QEMPO, we set \( \frac{1}{\lambda} \) to 4e-3. For QEMPO-KL, we set \( \frac{1}{\lambda_1} \) to 4e-3 and \( \frac{\lambda_2}{\lambda_1} \) to 1e-2. We employ xverify \citep{xverify} to validate whether the reasoning matches the given answer.

\textbf{Evaluation metrics.} we employ pass@k to evaluate both the quality and diversity of the model. For the n in pass@k, we set it to 100.

\subsection{Performance}

 \textbf{Quality constraints.} During our experiments, we observed that training QEMPO and QEMPO-KL could lead to an excessive increase in policy entropy, resulting in a sharp decline in reward. We attribute this phenomenon to the fact that entropy increase was encouraged even for low-quality outputs during training. This creates an issue where, when the reward term is zero, the loss can be reduced by increasing entropy. However, increasing entropy under low-quality outputs often leads to further degradation in policy output quality. To prevent over-optimization of entropy for low-quality outputs, we compute the variance term \( \mathrm{var}(\log \pi) \) only on data where the outputs are correct. We found that this modification significantly stabilizes the overall training process.

\textbf{Performance}. The experimental results are shown in \cref{fig:pass-k}. It can be observed that compared to RLHF, QEMPO and QEMPO-KL show little difference in pass@1 across various datasets. However, as k increases, the performance of QEMPO and QEMPO-KL begins to surpass that of RLHF. Moreover, on more challenging datasets, the improvement brought by QEMPO and QEMPO-KL becomes more pronounced as k increases. This indicates that for difficult problems, increasing diversity can also enhance model performance.

\begin{figure}[h]
    \centering
    \includegraphics[height=0.36\textwidth]{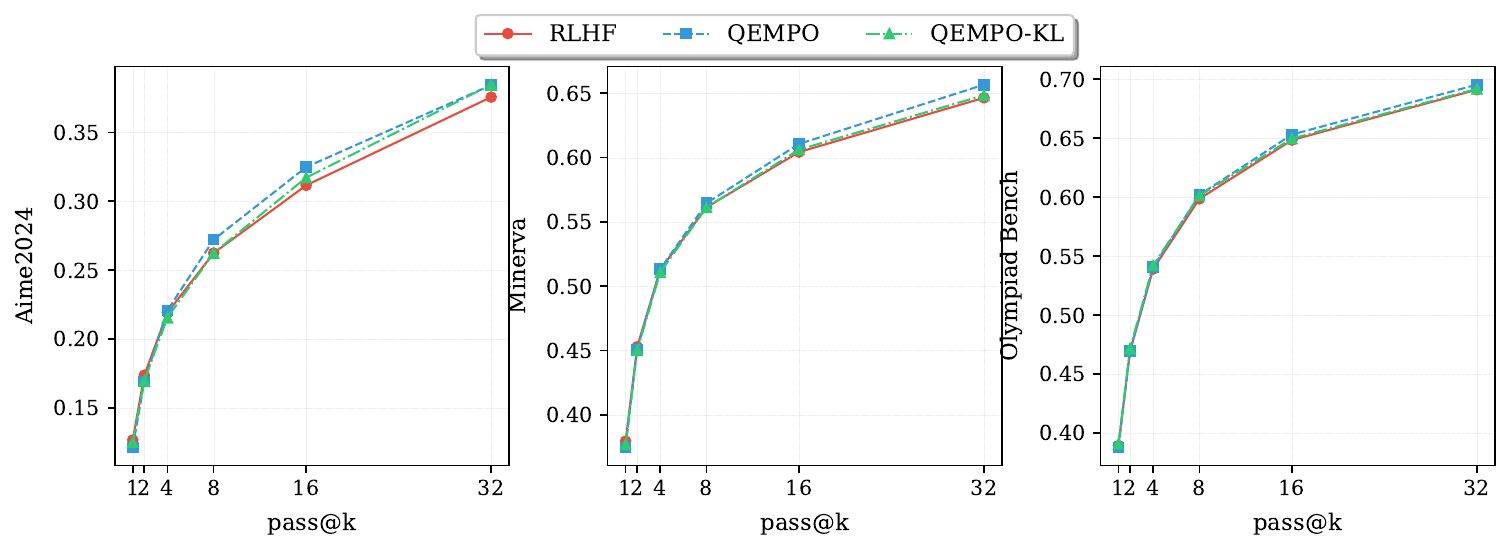}
    \caption{Pass@k results of RLHF, QEMPO, and QEMPO-KL on different datasets.}
    \label{fig:pass-k}
\end{figure}

\section{Related Work}
Existing research indicates that alignment methods while significantly improving the quality, helpfulness, and safety of LLM outputs, are widely observed to reduce the diversity of the generated content \citep{li2025preserving,murthy2024one,padmakumar2023does,slocum2025diverse,sun2025curiosity,xu2025understanding,kirk2023understanding}. 
%This decrease in diversity is manifested as a reduction in the lexical, syntactic, and conceptual variability of the model's outputs, leading to content that is homogeneous and repetitive.
In response to the decline in diversity attributed to alignment methodologies, researchers have introduced a range of strategies.
%\textbf{Sampling methods}. Sampling methods are a post-processing technique that doesn't require modifying a model's training data or training approach. Traditional methods, such as Temperature Sampling, Top-p Sampling, and Top-k Sampling, increase diversity by adjusting the randomness of token selection. However, these methods can lead to incoherent or repetitive outputs, especially when the temperature is set high. To overcome these limitations, new sampling methods have been developed, including min-p Sampling\citep{min-p}, chain-of-Thought Decoding\citep{cot-Decoding}, and contrastive decoding\citep{Contrastive-decoding}. 

\textbf{Diversity-oriented data augmentation}. The core idea behind these methodologies centers on enhancing the diversity of large language model (LLM) outputs through the strategic expansion and diversification of training datasets. Research has shown that combining data augmentation with SFT \citep{Diversity-oriented} and DPO \citep{Diverse-preference-optimization} can effectively boost the diversity of LLM-generated content.

\textbf{Diversity-Oriented Training Methods}. These methods aim to boost the diversity of an LLM's generated output by designing specific reward functions or optimization objectives. Soft Preference Learning \citep{slocum2025diverse} decouples the entropy and cross-entropy terms in the KL penalty, allowing for fine-grained control over the diversity of the LLM's output. GEM \citep{li2025preserving} uses game theory to design a novel SFT method that preserves the diversity of the model's output. Curiosity-Driven RLHF \citep{sun2025curiosity} is a framework that combines an intrinsic reward for novel states with the traditional sparse external reward during the RLHF training phase.  \textcolor{black}{The trade-off between diversity and quality is examined through the combination of negative log-likelihood training and temperature scaling in \citep{verine2025improving}.}

 \textbf{Entropy}. The intrinsic relationship between entropy and the diversity of LLM outputs is direct and significant. Generally, higher entropy in a model's predicted token distribution is associated with greater diversity in the generated text, as it allows the model to explore a wider range of token choices at each step \citep{senautomating}. Conversely, low-entropy distributions often lead to more deterministic, repetitive, and less diverse outputs \citep{zhou2024balancing}. However, increasing the entropy of the output distribution can sometimes result in incoherent outputs, thus lowering the quality of the generated text \citep{carlsson2024hyperfitting}.

\section{Conclusion}
In this work, we revisit the roles of quality and diversity in alignment tasks. Theoretically, we demonstrate that quality and diversity are two essential and complementary components of alignment. For RLHF, we prove that the optimal policy form coincides with that of a KL minimization problem subject to quality constraints. To further enhance the entropy of policies under this optimal form, we propose two novel methods: QEMPO and QEMPO-KL. QEMPO is designed to maximize policy entropy under quality constraints, while QEMPO-KL incorporates both quality and KL constraints to maximize entropy. Theoretically, we establish that under certain conditions, the policies derived from QEMPO and QEMPO-KL exhibit higher entropy than those obtained with RLHF. Through both offline and online experiments, we demonstrate that the proposed methods achieve improved diversity while maintaining competitive performance in quality relative to RLHF.

\bibliographystyle{iclr2024_conference}
\bibliography{iclr2024_conference}

@article{kirk2023understanding,
  title={Understanding the effects of rlhf on llm generalisation and diversity},
  author={Kirk, Robert and Mediratta, Ishita and Nalmpantis, Christoforos and Luketina, Jelena and Hambro, Eric and Grefenstette, Edward and Raileanu, Roberta},
  journal={arXiv preprint arXiv:2310.06452},
  year={2023}
}

@article{xu2025understanding,
  title={Understanding the Effects of RLHF on the Quality and Detectability of LLM-Generated Texts},
  author={Xu, Beining and Zubiaga, Arkaitz},
  journal={arXiv preprint arXiv:2503.17965},
  year={2025}
}

@article{padmakumar2023does,
  title={Does writing with language models reduce content diversity?},
  author={Padmakumar, Vishakh and He, He},
  journal={arXiv preprint arXiv:2309.05196},
  year={2023}
}

@article{murthy2024one,
  title={One fish, two fish, but not the whole sea: Alignment reduces language models' conceptual diversity},
  author={Murthy, Sonia K and Ullman, Tomer and Hu, Jennifer},
  journal={arXiv preprint arXiv:2411.04427},
  year={2024}
}

@article{sun2025curiosity,
  title={Curiosity-Driven Reinforcement Learning from Human Feedback},
  author={Sun, Haoran and Chai, Yekun and Wang, Shuohuan and Sun, Yu and Wu, Hua and Wang, Haifeng},
  journal={arXiv preprint arXiv:2501.11463},
  year={2025}
}

@inproceedings{slocum2025diverse,
  title={Diverse preference learning for capabilities and alignment},
  author={Slocum, Stewart and Parker-Sartori, Asher and Hadfield-Menell, Dylan},
  booktitle={The Thirteenth International Conference on Learning Representations},
  year={2025}
}

@inproceedings{li2025preserving,
  title={Preserving Diversity in Supervised Fine-Tuning of Large Language Models},
  author={Li, Ziniu and Chen, Congliang and Xu, Tian and Qin, Zeyu and Xiao, Jiancong and Luo, Zhi-Quan and Sun, Ruoyu},
  booktitle={ICLR},
  year={2025}
}

@article{Diversity-oriented,
  title={Diversity-oriented data augmentation with large language models},
  author={Wang, Zaitian and Zhang, Jinghan and Zhang, Xinhao and Liu, Kunpeng and Wang, Pengfei and Zhou, Yuanchun},
  journal={arXiv preprint arXiv:2502.11671},
  year={2025}
}

@article{Diverse-preference-optimization,
  title={Diverse preference optimization},
  author={Lanchantin, Jack and Chen, Angelica and Dhuliawala, Shehzaad and Yu, Ping and Weston, Jason and Sukhbaatar, Sainbayar and Kulikov, Ilia},
  journal={arXiv preprint arXiv:2501.18101},
  year={2025}
}

@article{rlhf,
  title={Training language models to follow instructions with human feedback},
  author={Ouyang, Long and Wu, Jeffrey and Jiang, Xu and Almeida, Diogo and Wainwright, Carroll and Mishkin, Pamela and Zhang, Chong and Agarwal, Sandhini and Slama, Katarina and Ray, Alex and others},
  journal={Advances in neural information processing systems},
  volume={35},
  pages={27730--27744},
  year={2022}
}

@article{dpo,
  title={Direct preference optimization: Your language model is secretly a reward model},
  author={Rafailov, Rafael and Sharma, Archit and Mitchell, Eric and Manning, Christopher D and Ermon, Stefano and Finn, Chelsea},
  journal={Advances in neural information processing systems},
  volume={36},
  pages={53728--53741},
  year={2023}
}

@article{Qwen3,
  title={Qwen3 technical report},
  author={Yang, An and Li, Anfeng and Yang, Baosong and Zhang, Beichen and Hui, Binyuan and Zheng, Bo and Yu, Bowen and Gao, Chang and Huang, Chengen and Lv, Chenxu and others},
  journal={arXiv preprint arXiv:2505.09388},
  year={2025}
}

@article{deepseek-r1,
  title={Deepseek-r1: Incentivizing reasoning capability in llms via reinforcement learning},
  author={Guo, Daya and Yang, Dejian and Zhang, Haowei and Song, Junxiao and Zhang, Ruoyu and Xu, Runxin and Zhu, Qihao and Ma, Shirong and Wang, Peiyi and Bi, Xiao and others},
  journal={arXiv preprint arXiv:2501.12948},
  year={2025}
}

@inproceedings{senautomating,
  title={Automating Evaluation of Creativity in LLMs with Semantic Entropy and Efficient Multi-Agent Judge},
  author={Sen, Tan Min and Chun, Zachary Choy Kit and Saikia, Swaagat Bikash and Alsagoff, Syed Ali Redha and Mohor, Banerjee and Wangsajaya, Nadya Yuki and Chan, Alvin},
  booktitle={Workshop on Reasoning and Planning for Large Language Models},
  year={2025}
}

@article{zhou2024balancing,
  title={Balancing diversity and risk in llm sampling: How to select your method and parameter for open-ended text generation},
  author={Zhou, Yuxuan and Keuper, Margret and Fritz, Mario},
  journal={arXiv preprint arXiv:2408.13586},
  year={2024}
}

@article{carlsson2024hyperfitting,
  title={The hyperfitting phenomenon: Sharpening and stabilizing llms for open-ended text generation},
  author={Carlsson, Fredrik and Liu, Fangyu and Ward, Daniel and Kurfali, Murathan and Nivre, Joakim},
  journal={arXiv preprint arXiv:2412.04318},
  year={2024}
}

@article{ultrafeedback,
  title={Ultrafeedback: Boosting language models with high-quality feedback},
  author={Cui, Ganqu and Yuan, Lifan and Ding, Ning and Yao, Guanming and Zhu, Wei and Ni, Yuan and Xie, Guotong and Liu, Zhiyuan and Sun, Maosong},
  year={2023}
}

@article{qwen2.5,
  title={Qwen2.5 Technical Report},
  author={Team, Qwen},
  journal={arXiv preprint arXiv:2412.15115},
  year={2024}
}

@article{dubey2024llama,
  title={The llama 3 herd of models},
  author={Dubey, Abhimanyu and Jauhri, Abhinav and Pandey, Abhinav and Kadian, Abhishek and Al-Dahle, Ahmad and Letman, Aiesha and Mathur, Akhil and Schelten, Alan and Yang, Amy and Fan, Angela and others},
  journal={arXiv e-prints},
  pages={arXiv--2407},
  year={2024}
}

@article{llama3.2,
  title={Llama 3.2: Revolutionizing edge ai and vision with open, customizable models},
  author={Meta, AI},
  journal={Meta AI Blog. Retrieved December},
  volume={20},
  pages={2024},
  year={2024}
}

@article{metric,
  title={Benchmarking linguistic diversity of large language models},
  author={Guo, Yanzhu and Shang, Guokan and Clavel, Chlo{\'e}},
  journal={arXiv preprint arXiv:2412.10271},
  year={2024}
}

@article{zephyr,
  title={Zephyr: Direct distillation of lm alignment},
  author={Tunstall, Lewis and Beeching, Edward and Lambert, Nathan and Rajani, Nazneen and Rasul, Kashif and Belkada, Younes and Huang, Shengyi and Von Werra, Leandro and Fourrier, Cl{\'e}mentine and Habib, Nathan and others},
  journal={arXiv preprint arXiv:2310.16944},
  year={2023}
}

@article{adam,
  title={Adam: A method for stochastic optimization},
  author={Kingma, Diederik P},
  journal={arXiv preprint arXiv:1412.6980},
  year={2014}
}

@article{gvpo,
  title={GVPO: Group variance policy optimization for large language model post-training},
  author={Zhang, Kaichen and Hong, Yuzhong and Bao, Junwei and Jiang, Hongfei and Song, Yang and Hong, Dingqian and Xiong, Hui},
  journal={arXiv preprint arXiv:2504.19599},
  year={2025}
}

@article{xverify,
  title={xverify: Efficient answer verifier for reasoning model evaluations},
  author={Chen, Ding and Yu, Qingchen and Wang, Pengyuan and Zhang, Wentao and Tang, Bo and Xiong, Feiyu and Li, Xinchi and Yang, Minchuan and Li, Zhiyu},
  journal={arXiv preprint arXiv:2504.10481},
  year={2025}
}

@article{gsm8k,
  title={Training verifiers to solve math word problems},
  author={Cobbe, Karl and Kosaraju, Vineet and Bavarian, Mohammad and Chen, Mark and Jun, Heewoo and Kaiser, Lukasz and Plappert, Matthias and Tworek, Jerry and Hilton, Jacob and Nakano, Reiichiro and others},
  journal={arXiv preprint arXiv:2110.14168},
  year={2021}
}

@inproceedings{math500,
  title={Let's verify step by step},
  author={Lightman, Hunter and Kosaraju, Vineet and Burda, Yuri and Edwards, Harrison and Baker, Bowen and Lee, Teddy and Leike, Jan and Schulman, John and Sutskever, Ilya and Cobbe, Karl},
  booktitle={The Twelfth International Conference on Learning Representations},
  year={2023}
}

@article{entrpo,
  title={EnTRPO: trust region policy optimization method with entropy regularization},
  author={Roostaie, Sahar and Ebadzadeh, Mohammad Mehdi},
  journal={arXiv preprint arXiv:2110.13373},
  year={2021}
}

@article{ERC-TRPO,
  title={Trust region policy optimization via entropy regularization for Kullback--Leibler divergence constraint},
  author={Xu, Haotian and Xuan, Junyu and Zhang, Guangquan and Lu, Jie},
  journal={Neurocomputing},
  volume={589},
  pages={127716},
  year={2024},
  publisher={Elsevier}
}

@article{pg,
  title={Policy gradient methods for reinforcement learning with function approximation},
  author={Sutton, Richard S and McAllester, David and Singh, Satinder and Mansour, Yishay},
  journal={Advances in neural information processing systems},
  volume={12},
  year={1999}
}

@article{simpo,
  title={Simpo: Simple preference optimization with a reference-free reward},
  author={Meng, Yu and Xia, Mengzhou and Chen, Danqi},
  journal={Advances in Neural Information Processing Systems},
  volume={37},
  pages={124198--124235},
  year={2024}
}

@article{verine2025improving,
  title={Improving diversity in language models: When temperature fails, change the loss},
  author={Verine, Alexandre and Bronnec, Florian Le and Zheng, Kunhao and Allauzen, Alexandre and Chevaleyre, Yann and Negrevergne, Benjamin},
  journal={arXiv preprint arXiv:2508.09654},
  year={2025}
}

@inproceedings{zhang2021trading,
  title={Trading off diversity and quality in natural language generation},
  author={Zhang, Hugh and Duckworth, Daniel and Ippolito, Daphne and Neelakantan, Arvind},
  booktitle={Proceedings of the Workshop on Human Evaluation of NLP Systems (HumEval)},
  pages={25--33},
  year={2021}
}

@article{shypula2025evaluating,
  title={Evaluating the diversity and quality of llm generated content},
  author={Shypula, Alexander and Li, Shuo and Zhang, Botong and Padmakumar, Vishakh and Yin, Kayo and Bastani, Osbert},
  journal={arXiv preprint arXiv:2504.12522},
  year={2025}
}

@article{mmlu,
  title={Measuring massive multitask language understanding},
  author={Hendrycks, Dan and Burns, Collin and Basart, Steven and Zou, Andy and Mazeika, Mantas and Song, Dawn and Steinhardt, Jacob},
  journal={arXiv preprint arXiv:2009.03300},
  year={2020}
}

\appendix

\section{Reproducibility statement}
All theoretical proofs are provided in the appendix. The training data and models designed in this experiment are publicly available. We use trl to implement the offline model experiments for QEMPO and QEMPO-KL. We use verl to implement the online model experiments for QEMPO and QEMPO-KL. To facilitate reproducibility, we provide the following simplified implementation code.
\begin{lstlisting}[caption={A Simple QEMPO-KL Code Implementation in offline mode}, label={lst:policy_loss}]
def compute_qempo_kl_loss(beta1, beta2, chosen_logps, rejected_logps, ref_chosen_logps, ref_rejected_logps, **kwargs):
    # beta1: \frac{1}{\lambda_{1}} in QEMPO-KL
    # beta2:  \frac{\lambda_{2}}{\lambda_{1}} in QEMPO-KL
    chosen_logratios = beta1 * chosen_logps + beta2 * (chosen_logps - ref_chosen_logps)
    rejected_logratios = beta1 * rejected_logps + beta2 * (rejected_logps - ref_rejected_logps)
    logits = chosen_logratios - rejected_logratios
    return -F.logsigmoid(logits)
\end{lstlisting}
\begin{lstlisting}[caption={A Simple QEMPO Code Implementation in offline mode}, label={lst:policy_loss}]
def compute_qempo_loss(beta, chosen_logps, rejected_logps, **kwargs):
    # beta: \frac{1}{\lambda} in QEMPO
    chosen_logratios = beta * chosen_logps
    rejected_logratios = beta * rejected_logps
    logits = chosen_logratios - rejected_logratios
    return -F.logsigmoid(logits)
\end{lstlisting}
\begin{lstlisting}[caption={A Simple QEMPO-KL Code Implementation within a single group in online mode}, label={lst:policy_loss}]
def compute_qempo_policy_loss(beta1, beta2, advantages, reward, old_log_prob, log_prob, eos_mask, **kwargs):
    # beta1: \frac{1}{\lambda_{1}} in QEMPO-KL
    # beta2:  \frac{\lambda_{2}}{\lambda_{1}} in QEMPO
    score = log_prob - torch.mean(log_prob)
    score_old = old_log_prob - torch.mean(old_log_prob)
    adv = torch.mean(torch.sum(-advantages * log_prob * eos_mask, dim=1))
    cov = beta2 * torch.mean(score * score_old)
    if sum(reward).item() == len(reward_list):
        var = torch.mean(score * score)
    else:
        var = torch.tensor(0.0)
    return 2 * (beta1 + beta2) * (adv + beta2 * cov + (beta1 + beta2) * var)
\end{lstlisting}
\begin{lstlisting}[caption={A Simple QEMPO Code Implementation within a single group in online mode}, label={lst:policy_loss}]
def compute_qempo_policy_loss(beta, advantages, reward, log_prob, eos_mask, **kwargs):
    # beta: \frac{1}{\lambda} in QEMPO
    score = log_prob - torch.mean(log_prob)
    adv = torch.mean(torch.sum(-advantages * log_prob * eos_mask, dim=1))
    if sum(reward).item() == len(reward_list):
        var = torch.mean(score * score)
    else:
        var = torch.tensor(0.0)
    return 2 * beta * (adv + beta * var)
\end{lstlisting}

\section{Discussion}
\label{sec:diss}
We briefly discuss the distinctions between our method and several existing approaches in this section. Soft Preference Learning \citep{slocum2025diverse} decouples the entropy and cross-entropy terms in the KL penalty, allowing for fine-grained control over the diversity of the LLM’s output.  Unlike SPL, we propose a distinct optimization objective—quality-constrained entropy maximization. 

\textcolor{black}{SimPO \citep{simpo} employs the average log probability of a sequence as an implicit reward to eliminate dependency on a reference model. In contrast, QEMPO removes the reference model with the aim of increasing the policy's entropy.}

EnTRPO\citep{entrpo} combines entropy with the reward value as the optimization objective while incorporating KL divergence as a constraint. ERC-TRPO\citep{ERC-TRPO} treats the difference between KL divergence and entropy as a constraint and maximizes the reward as the objective. In contrast to both EnTRPO and ERC-TRPO, QEMPO-KL exclusively optimizes for entropy alone, subject to both quality and KL constraints. QEMPO, meanwhile, further simplifies this framework by entirely omitting the KL constraint.

\section{Performance on MMLU}
\begin{table}[htbp]
\centering
\caption{MMLU evaluation results. QEMPO and QEMPO-KL do not negatively affect general model utility.}
\label{tab:mmlu}
\begin{tabular}{lc}
\hline
\textbf{Model} & \textbf{MMLU} \\
\hline
Llama-3.1-8B-Instruct & 66.31 \\
$\quad \vdash$ QEMPO & \textbf{66.42} \\
$\quad \llcorner$ QEMPO-KL & 66.37 \\
\hline
Llama-3.2-1B-Instruct & 46.21 \\
$\quad \vdash$ QEMPO & 46.23 \\
$\quad \llcorner$QEMPO-KL & \textbf{46.28} \\
\hline
Qwen2.5-7B-Instruct & 74.12 \\
$\quad \vdash$QEMPO & \textbf{74.16} \\
$\quad \llcorner$ QEMPO-KL & {74.14} \\
\hline
Qwen2.5-1.5B-Instruct & {60.46} \\
$\quad \vdash$ QEMPO & \textbf{60.53} \\
$\quad \llcorner$ QEMPO-KL & {60.49} \\
\hline
\end{tabular}
\end{table}

\section{Proof of \cref{prop:kl-rlhf}}
\label{proof:p1}
\propone*
\begin{proof}
We first transform the constraint equivalently, so the original optimization objective can be converted to:
\begin{align}
    \min_{\pi} \mathrm{KL}[\pi(y|x) || \pi_{\mathrm{ref}}(y|x)] \quad s.t. \quad \mathrm{R}-\mathbb{E}_{\pi(y|x)}[r(x,y)] \le 0
\end{align}
According to the method of Lagrange multipliers, the following functional can be constructed:
\begin{align}
    L(\pi(y|x)) = \sum_{y} \pi(y|x) \ln\frac{\pi(y|x)}{\pi_{ref}(y|x)} +\lambda(\mathrm{R} -\sum_{y}\pi(y|x) r(x,y) )
\end{align}
where $\lambda>0$.
Since KL divergence is a convex function for $\pi(y|x)$, its minimum is achieved when the gradient is zero. Setting the derivative of $L(\pi(y|x))$ with respect to $\pi(y|x)$ to zero, we get:
\begin{align}
    \frac{\partial L(\pi(y|x))}{ \partial\pi(y|x)} = \ln \pi(y|x)+1-\ln \pi_{ref}(y|x) -  \lambda r(x,y) =0.
\end{align}
Simplifying the above expression yields:
\begin{align} 
  \pi (y|x)=  \pi_{\mathrm{ref}}(y|x) \exp( \lambda r(x,y)-1).
\end{align}
Since $\pi(y|x)$ is an LLM, it can be regarded as a distribution. Furthermore, we can obtain:
\begin{align} 
  \sum_{y}\pi (y|x)=  \sum_{y}\pi_{\mathrm{ref}}(y|x) \exp( \lambda r(x,y)-1)=1
\end{align}
Finally, we can get: 
\begin{align}
 \pi (y|x) &= \frac{\pi_{\mathrm{ref}}(y|x) \exp( \lambda r(x,y)-1)}{1}
 \\
 &=\frac{\pi_{\mathrm{ref}}(y|x) \exp( \lambda r(x,y)-1)}{ \sum_{y}\pi_{\mathrm{ref}}(y|x) \exp( \lambda r(x,y)-1)} 
 \\
 &=\frac{\pi_{\mathrm{ref}}(y|x) \exp( \lambda r(x,y))}{ \sum_{y}\pi_{\mathrm{ref}}(y|x) \exp( \lambda r(x,y))}
 \\
 &= \frac{\pi_{\mathrm{ref}}(y|x) \exp( \lambda r(x,y))}{ Z(x)}
\end{align}
where $Z(x)=\sum_{y}\pi_{\mathrm{ref}}(y|x) \exp( \lambda r(x,y))$.
\end{proof}

\section{Proof of \cref{prop:one}}
\label{proof:2}
\propOne*

\begin{proof}
We transform the optimization objective of minimizing KL divergence into a maximization form:
\begin{align}
\min_{\pi} \mathrm{KL}(\pi(y|x) || \pi^{*}(y|x)) &=\max_{\pi} - \mathrm{KL}(\pi(y|x) || \pi^{*}(y|x)) \\
&=
\max_{\pi}  \underbrace{\sum _{y} -\pi(y|x)\ln \pi(y|x)}_{H_{\pi}(Y|X)} +  \sum _{y} \pi(y|x)\ln \pi^{*}(y|x)
\end{align}
For \(\sum_{y} \pi(y|x) \ln \pi^{*}(y|x)\), we can expand it as:
\begin{align}
    \sum _{y} \pi(y|x)\ln \pi^{*}(y|x) = \sum _{y\in Y^{+}} \pi(y|x)\ln \pi^{*}(y|x) +  \sum _{y\in Y^{-}} \pi(y|x)\ln \pi^{*}(y|x).
\end{align}
According to the two attributes of \(\pi^{*}(y|x)\) (quality and diversity), together with the assumption that \(\pi^{*}\) is uniform over \(Y^{-}\), we can infer that
\begin{align}
    \pi^{*}(y|x) = \left\{\begin{matrix}
 \frac{1- \epsilon}{|Y^{+}|} & y\in Y^{+} \\
  \frac{\epsilon}{|Y^{-}|} & y\in Y^{-} &
\end{matrix}\right.
\end{align}
where $\frac{1- \epsilon}{|Y^{+}|}  \gg \frac{\epsilon}{|Y^{-}|}$. 
Furthermore, we can obtain:
\begin{align}
\sum _{y} \pi(y|x)\ln \pi^{*}(y|x) 
&=\sum _{y\in Y^{+}} \pi(y|x)\ln \frac{1- \epsilon}{|Y^{+}|}   +  \sum _{y\in Y^{-}} \pi(y|x)  \ln \frac{\epsilon}{|Y^{-}|} 
\end{align}
Let \(P_{w} = \sum_{y \in Y^{+}} \pi(y|x) \in (0,1)\). Substituting this into the equation above, we obtain:
\begin{align}
\sum _{y} \pi(y|x)\ln \pi^{*}(y|x) 
&=P_{w}\ln \frac{1- \epsilon}{|Y^{+}|}   + (1-P_{w}) \ln \frac{\epsilon}{|Y^{-}|} 
\\
&=P_{w}(\ln \frac{1- \epsilon}{|Y^{+}|} - \ln \frac{\epsilon}{|Y^{-}|} ) + \ln \frac{\epsilon}{|Y^{-}|} 
\end{align}
Since $\ln \frac{1- \epsilon}{|Y^{+}|} - \ln \frac{\epsilon}{|Y^{-}|} >0$, it can be shown that \(\sum_{y} \pi(y|x)\ln \pi^{*}(y|x)\) increases as \(P_{w}\) increases. Therefore, improving the output quality of \(\pi(y|x)\) can increase \(\sum_{y} \pi(y|x)\ln \pi^{*}(y|x)\), i.e.,
\begin{align}
    \max_{\pi} \sum _{y} \pi(y|x)\ln \pi^{*}(y|x) 
&\propto \max P_{w} =\max_{\pi}\sum _{y\in Y^{+}} \pi(y|x) .
\end{align}
For diversity, we can use \(H_{\pi}(Y|X)\) to represent it. This is because when the output distribution of \(\pi(y|x)\) is more uniform, \(H_{\pi}(Y|X)\) becomes larger.

In summary, the alignment objective $ \mathrm{KL}(\pi(y|x) || \pi^{*}(y|x))$ can be minimized by maximizing both quality and diversity.

\end{proof}

\section{Proof of Corollary \ref{cly:one}}
\label{Corollary:3.1}
\begin{proof}
For Policy Gradient, the optimization objective is:
\begin{align}
    \max_{\pi}\sum _{y} r(x,y) \pi(y|x).
\end{align}
We can further express the above equation as:
\begin{align}
    \max_{\pi}\sum _{y\in Y } r(x,y) \pi(y|x)= \max_{\pi}\sum _{y\in Y^{+}} r(x,y) \pi(y|x) + \max_{\pi}\sum _{y\in Y^{-}} r(x,y) \pi(y|x).
\end{align}
If the reward function is defined as $r(x,y) =\left\{\begin{matrix}
 1 & y \in Y^{+}\\
 0 & y \in Y^{-}
\end{matrix}\right.$, then 
\begin{align}
    \max_{\pi}\sum _{y\in Y } r(x,y) \pi(y|x) &= \max_{\pi}\sum _{y\in Y^{+}} 1* \pi(y|x) + \max_{\pi}\sum _{y\in Y^{-}} 0* \pi(y|x) \\
    &= \max_{\pi}\sum _{y\in Y^{+}} \pi(y|x)
    .
\end{align}
Based on the above analysis, we can conclude that Policy Gradient only optimizes the quality component in alignment tasks.
\end{proof}

\section{Proof of \cref{prop:kl-qempo}}
\label{p3}
\propTwo*
\begin{proof}
We first convert the original optimization problem from maximization to minimization and equivalently transform the constraints, so the original optimization objective can be transformed into:
\begin{align}
    \min_{\pi} -H_{\pi}(Y|X) \quad s.t.  \left\{\begin{matrix}
 \mathrm{R} - \mathbb{E}_{\pi(y|x)}[r(x,y)] \le 0  
\\
\mathrm{KL}(\pi_{}\parallel \pi_{\mathrm{ref}} )- \mathrm{K} \le 0
\end{matrix}\right..
\end{align}
According to the method of Lagrange multipliers, the following functional can be constructed:
\begin{align}
    L(\pi(y|x)) = \sum_{y} \pi(y|x) \ln \pi(y|x) +\lambda_{1}(\mathrm{R} -\sum_{y}\pi(y|x) r(x,y) ) + \lambda_{2}(\sum_{y} \pi(y|x) \ln\frac{\pi(y|x)}{\pi_{ref}(y|x)}-\mathrm{K})
\end{align}
where $\lambda_{1}>0$ and $\lambda_{2}>0$. Since $-H_{\pi}(Y|X) $ is a convex function for $\pi(y|x)$, its minimum is achieved when the gradient is zero. Setting the derivative of $L(\pi(y|x))$ with respect to $\pi(y|x)$ to zero, we get:
\begin{align}
\frac{\partial L(\pi(y|x))}{ \partial\pi(y|x)} = \ln \pi(y|x)+1-  \lambda_{1} r(x,y) + \lambda_{2} \ln \pi(y|x)+1 - \lambda_{2}\pi_{ref}(y|x)=0.
\end{align}
Simplifying the above equation yields:
\begin{align}
    \pi(y|x)=\pi_{\mathrm{ref}}^{\frac{\lambda_{2}}{\lambda_{2}+1}} \exp( \frac{\lambda_{1}}{\lambda_{2}+1} r(x,y)-2).
\end{align}
Since $\pi(y|x)$ is an LLM, it therefore satisfies: 
\begin{align}
    \sum_{y}\pi(y|x) = 1.
\end{align}
Furthermore, we can obtain:
\begin{align}
 \pi (y|x) &= \frac{\pi_{\mathrm{ref}}^{\frac{\lambda_{2}}{\lambda_{2}+1}} \exp( \frac{\lambda_{1}}{\lambda_{2}+1} r(x,y)-2)}{1}
 \\
 &=\frac{\pi_{\mathrm{ref}}^{\frac{\lambda_{2}}{\lambda_{2}+1}} \exp( \frac{\lambda_{1}}{\lambda_{2}+1} r(x,y)-2)}{ \sum_{y}\pi_{\mathrm{ref}}^{\frac{\lambda_{2}}{\lambda_{2}+1}} \exp( \frac{\lambda_{1}}{\lambda_{2}+1} r(x,y)-2)} 
 \\
 &=\frac{\pi_{\mathrm{ref}}^{\frac{\lambda_{2}}{\lambda_{2}+1}} \exp( \frac{\lambda_{1}}{\lambda_{2}+1} r(x,y))}{ \sum_{y}\pi_{\mathrm{ref}}^{\frac{\lambda_{2}}{\lambda_{2}+1}} \exp( \frac{\lambda_{1}}{\lambda_{2}+1} r(x,y))}
 \\
 &= \frac{\pi_{\mathrm{ref}}^{\frac{\lambda_{2}}{\lambda_{2}+1}} \exp( \frac{\lambda_{1}}{\lambda_{2}+1} r(x,y))}{ Z(x)}
\end{align}
where $Z(x)= \sum_{y}\pi_{\mathrm{ref}}^{\frac{\lambda_{2}}{\lambda_{2}+1}} \exp( \frac{\lambda_{1}}{\lambda_{2}+1} r(x,y))$.
\end{proof}

\section{Proof of \cref{prop:entropy-kl}}
\label{p4}
\begin{lemma}\label{lemma1}

\textcolor{black}{
Let \(\mathbf{z} = (z_1, z_2, \dots, z_n)\) be an arbitrary real-valued vector, and define  
$
p_i(s) = \frac{\exp({sz_i})}{\sum_{j=1}^n \exp({sz_j})}=\frac{\exp({z_i})^{s}}{\sum_{j=1}^n \exp({z_j})^{s}},
$  
and  
$  
H(s) = - \sum_{i=1}^n p_i(s) \ln p_i(s).
$   
Then for any \(s_2 > s_1 > 0\), the following holds:
\[
H(s_2) \le H(s_1),
\]
with equality if and only if all components of \(\mathbf{z}\) are equal.}
\end{lemma}

\begin{proof}
Let \( Z(s) = \sum_{j=1}^n e^{s z_j} \), so that \( p_i(s) = \exp({s z_i}) / Z(s) \). The derivative of \( Z(s) \) with respect to \( s \) is:
\begin{align}
    Z'(s) = \sum_{j=1}^n z_j \exp({s z_j}).
\end{align}
We can compute the derivative of \( p_i(s) \) with respect to \( s \) as:
\begin{align}
    p_i'(s) &= \frac{z_i \exp({s z_i}) Z(s) - \exp({s z_i}) Z'(s)}{Z(s)^2}
    \\
&= \frac{z_i \exp({s z_i})}{Z(s)} - \frac{\exp({s z_i})}{Z(s)} \cdot \frac{Z'(s)}{Z(s)}
\\
&= p_i(s) \left[ z_i - \frac{Z'(s)}{Z(s)}. \right]
\end{align}
Note that \( \frac{Z'(s)}{Z(s)} = \sum_{j=1}^n z_j p_j(s) = \mathbb{E}_{p(s)}[z] \). Therefore:
\begin{align}
    p_i'(s) = p_i(s) \left( z_i - \mathbb{E}[z] \right)
\end{align}
where $\mathbb{E}[z] = \sum_j z_j p_j(s)$.
The derivative of \( H(s) \) with respect to \( s \) can be computed as:
\begin{align}
    H'(s) = - \sum_i \left[ p_i'(s) \ln p_i(s) + p_i(s) \cdot \frac{p_i'(s)}{p_i(s)} \right]
= - \sum_i p_i'(s) \ln p_i(s) - \sum_i p_i'(s)
\end{align}
Note that $\sum_i p_i'(s) = \frac{d}{ds} \sum_i p_i(s) = 0$. Therefore:
\begin{align}
    H'(s) = - \sum_i p_i'(s) \ln p_i(s).
\end{align}
Substituting $p_i'(s) = p_i(s)(z_i - \mathbb{E}[z])$ and $\ln p_i(s) = s z_i - \ln Z(s)$ yields:
\begin{align}
    H'(s) &= - \sum_i p_i(s)(z_i - \mathbb{E}[z]) \left[ s z_i - \ln Z(s) \right] \\
&= - s \sum_i p_i(s)(z_i - \mathbb{E}[z]) z_i + \ln Z(s) \sum_i p_i(s)(z_i - \mathbb{E}[z]).
\end{align}
For the expression \( \sum_i p_i(s) (z_i - \mathbb{E}[z]) z_i \), we have:
\begin{align}
\sum_i p_i(s)(z_i - \mathbb{E}[z]) z_i
&= \sum_i p_i(s)(z_i - \mathbb{E}[z])(z_i - \mathbb{E}[z] + \mathbb{E}[z])
\\
&= \sum_i p_i(s)(z_i - \mathbb{E}[z])^2 + \mathbb{E}[z] \sum_i p_i(s)(z_i - \mathbb{E}[z]).
\end{align}
Since \( \sum_i p_i(s) (z_i - \mathbb{E}[z]) = 0 \), it follows that:
\begin{align}
    - s \sum_i p_i(s)(z_i - \mathbb{E}[z]) z_i = -s\mathrm{Var}_{p(s)}[z]
\end{align}
\begin{align}
    \ln Z(s) \sum_i p_i(s)(z_i - \mathbb{E}[z]) = 0.
\end{align}
Therefore, we obtain:
\begin{align}
    H'(s) = - s \cdot \mathrm{Var}_{p(s)}[z].
\end{align}
When not all components of \(\mathbf{z}\) are equal, we have \(\mathrm{Var}_{p(s)}[z] > 0\). Since \(s > 0\), it follows that:
\begin{align}
    H'(s) \le 0.
\end{align}
This implies that \( H(s) \) is monotonically decreasing for \( s > 0 \). Hence, for any \( s_2 > s_1 > 0 \):
\begin{align}
    H(s_{2}) < H(s_{1}).
\end{align}
When all components of \(\mathbf{z}\) are equal, we have \(\mathrm{Var}_{p(s)}[z] = 0\), and therefore:
\begin{align}
    H'(s) = 0.
\end{align}
This implies that for any \( s_2 > s_1 > 0 \):
\begin{align}
    H(s_{2}) =H(s_{1}).
\end{align}

Combining the above, we complete the proof of the lemma.
\end{proof}

\propFour*
\begin{proof}
We first transform \(\pi_{\mathrm{QEMPO-KL}}\) as follows:
\begin{align} 
  \pi_{\mathrm{QEMPO-KL}}(y|x)
  &=  \frac{\pi_{\mathrm{ref}}(y|x)^{\lambda_{2} /(\lambda_{2}+1)} \exp( \frac{\lambda_{1}}{\lambda_{2}+1} r(x,y))}{Z(x)} 
\\
  &= \frac{\pi_{\mathrm{ref}}(y|x) \exp( \frac{\lambda_{1}}{\lambda_{2}} r(x,y))^{\lambda_{2} /(\lambda_{2}+1)}}{Z(x)}
\\
&=\frac{ \exp( \ln \pi_{\text{ref}}(y|x)+\frac{\lambda_{1}}{\lambda_{2}} r(x,y))^{\lambda_{2} /(\lambda_{2}+1)}}{Z(x)}.
\end{align}

\textcolor{black}{The expression for \(\pi_{\text{RLHF}} (y|x)\) is:
\begin{align}
    \pi_{\text{RLHF}} (y|x)= \frac{\exp( \ln \pi_{ref}(y|x) +\frac{1}{\beta}r(x,y))}{Z(x)}.
\end{align}
Since \(\lambda_2 > 0\), it follows that \(\lambda_2 / (\lambda_2 + 1) < 1\). When \(\frac{1}{\beta} = \frac{\lambda_1}{\lambda_2}\), applying Lemma \ref{lemma1} yields:
\begin{align}
    H_{\pi_{\mathrm{QEMPO-KL}}}(Y|X) \ge  H_{\pi_{\mathrm{RLHF}}}(Y|X).
\end{align}
}

\end{proof}

\section{Proof of \cref{prop:qempo}}
\label{p5}
\propFive*
\begin{proof}
 We first convert the original optimization problem from maximization to minimization and equivalently transform the constraints, so the original optimization objective can be transformed into:
 \begin{align}
 \min_{\pi} - H_{\pi}(Y|X) \quad s.t. \quad \mathrm{R}-\mathbb{E}_{\pi(y|x)}[r(x,y)] \le 0 
 \end{align}
According to the method of Lagrange multipliers, the following functional can be constructed:
\begin{align}
    L(\pi(y|x)) = \sum_{y} \pi(y|x) \ln \pi(y|x) +\lambda(\mathrm{R} -\sum_{y}\pi(y|x) r(x,y) )
\end{align}
where $\lambda>0$. Since $-H_{\pi}(Y|X) $ is a convex function for $\pi(y|x)$, its minimum is achieved when the gradient is zero. Setting the derivative of $L(\pi(y|x))$ with respect to $\pi(y|x)$ to zero, we get:
\begin{align}
\frac{\partial L(\pi(y|x))}{ \partial\pi(y|x)} = \ln \pi(y|x)+1-  \lambda r(x,y) =0.
\end{align}
Simplifying the above equation yields:
\begin{align}
    \pi(y|x)= \exp(\lambda r(x,y)-1).
\end{align}
Since $\pi(y|x)$ is an LLM, it therefore satisfies: 
\begin{align}
    \sum_{y}\pi(y|x) = 1.
\end{align}

Furthermore, we can obtain:
\begin{align}
 \pi (y|x) &= \frac{ \exp(\lambda r(x,y)-1)}{1}
 \\
 &=\frac{ \exp(\lambda r(x,y)-1)}{ \sum_{y} \exp(\lambda r(x,y)-1)} 
 \\
 &=\frac{ \exp(\lambda r(x,y))}{ \sum_{y} \exp(\lambda r(x,y))}
 \\
 &= \frac{\exp(\lambda r(x,y))}{ Z(x)}
\end{align}
where $Z(x)= \sum_{y}\exp(\lambda r(x,y))$.

\end{proof}

\section{Proof of \cref{prop:entropy-kl-2} }
\label{p6}
\propSix*
\begin{proof}
We first transform \(\pi_{\mathrm{QEMPO-KL}}(y|x)\) and \(\pi_{\mathrm{QEMPO}}(y|x)\) respectively as follows:
\begin{align} 
   \pi_{\mathrm{QEMPO-KL}}(y|x)
  &=  \frac{\pi_{\mathrm{ref}}(y|x)^{\lambda_{2} /(\lambda_{2}+1)} \exp( \frac{\lambda_{1}}{\lambda_{2}+1} r(x,y))}{Z(x)} 
\\
  &= \frac{\pi_{\mathrm{ref}}(y|x) \exp( \frac{\lambda_{1}}{\lambda_{2}} r(x,y))^{\lambda_{2} /(\lambda_{2}+1)}}{Z(x)}
\\
&=\frac{ \exp( \ln \pi_{\text{ref}}(y|x)+\frac{\lambda_{1}}{\lambda_{2}} r(x,y))^{\lambda_{2} /(\lambda_{2}+1)}}{Z(x)}.
\\
  \pi_{\mathrm{QEMPO}}(y|x)
  &=  \frac{ \exp( \lambda r(x,y))}{Z(x)} 
\\
  &= \frac{\exp( \frac{\lambda_{1}}{\lambda_{2}} r(x,y))^{\frac{\lambda_{2}}{\lambda_{1}} \lambda}}{Z(x)} 
\end{align}
%
% For any two given outputs \(y_i\) and \(y_j\), we assume \(\pi_{\mathrm{QEMPO-KL}}(y_i|x) \ge \pi_{\mathrm{QEMPO-KL}}(y_j|x)\). The probability ratio of the two outputs is:
% \begin{align}
%     \frac{\pi_{\mathrm{QEMPO-KL}}(y_{i}|x)}{\pi_{\mathrm{QEMPO-KL}}(y_{j}|x)} =\frac{(\pi_{\mathrm{ref}}(y_{i}|x) \exp( \frac{\lambda_{1}}{\lambda_{2}} r(x,y_{i})))^{\lambda_{2} /(\lambda_{2}+1)}}{(\pi_{\mathrm{ref}}(y_{j}|x) \exp( \frac{\lambda_{1}}{\lambda_{2}} r(x,y_{j})))^{\lambda_{2} /(\lambda_{2}+1)}} \ge 1
% \end{align}
%
When \(\frac{\lambda_2}{\lambda_1} \lambda\) is sufficiently small, i.e., \(\frac{\lambda_2}{\lambda_1} \lambda \to 0\), we have:
\begin{align}
    \pi_{\mathrm{ref}}(y|x) ^{\frac{\lambda_{2}}{\lambda_{1}} \lambda} \to 1
\end{align}
\textcolor{black}{
At this point, \(\pi_{\mathrm{QEMPO}}(y|x)\) is expressed as:
\begin{align} 
  \pi_{\mathrm{QEMPO}}(y|x)
=\frac{ \exp(\ln \pi_{\mathrm{ref}}(y|x)+ \frac{\lambda_{1}}{\lambda_{2}} r(x,y)))^{^{\frac{\lambda_{2}}{\lambda_{1}} \lambda}}}{Z(x)} 
\end{align}
When \(\lambda \leq \frac{\lambda_1}{\lambda_2 + 1}\), it means \(\frac{\lambda_2}{\lambda_1} \lambda \leq \frac{\lambda_2}{\lambda_2 + 1}\). Applying Lemma \ref{lemma1} yields:
\begin{align}
    H_{\pi_{\mathrm{QEMPO}}}(Y|X) \geq H_{\pi_{\mathrm{QEMPO-KL}}}(Y|X).
\end{align}
}

\end{proof}

\end{document}